\documentclass{article}

\usepackage{arxiv}

\usepackage[utf8]{inputenc}
\usepackage[T1]{fontenc}
\usepackage{url}
\usepackage{booktabs}
\usepackage{nicefrac}
\usepackage{microtype}
\usepackage{graphicx}
\usepackage{subcaption}
\usepackage{makecell}
\usepackage{threeparttable}
\usepackage{tabularx}

\usepackage{amsmath}
\usepackage{amsfonts}
\usepackage{amssymb}
\usepackage{amsbsy}
\usepackage{amsthm}
\usepackage{algorithm}
\usepackage{algorithmic}

\usepackage{natbib}
\usepackage[pdftex,colorlinks=true,urlcolor=blue,citecolor=black,anchorcolor=black,linkcolor=black]{hyperref}
\usepackage{doi}

\newtheorem{theorem}{Theorem}

\newtheorem{proposition}[theorem]{Proposition}

\title{ANNEALED ENTROPIC ALLOCATION FOR RANKING AND SELECTION}

\author{
\href{https://orcid.org/0000-0003-4817-0773}{\includegraphics[scale=0.06]{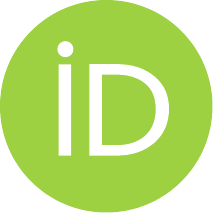}\hspace{1mm}Xin Fei} \\
Business School\\
The University of Edinburgh, Edinburgh, UK \\
\texttt{xfei@ed.ac.uk} \\
\And
\href{https://orcid.org/0000-0002-4343-5878}{\includegraphics[scale=0.06]{orcid.pdf}\hspace{1mm}Juergen Branke} \\
Warwick Business School\\
The University of Warwick, Coventry, UK\\
\texttt{Juergen.Branke@wbs.ac.uk}
}

\renewcommand{\shorttitle}{Annealed Entropic Allocation}

\hypersetup{
pdftitle={ANNEALED ENTROPIC ALLOCATION FOR RANKING AND SELECTION},
pdfsubject={Ranking and selection; pure exploration; large deviations},
pdfauthor={Fei, Branke},
pdfkeywords={Ranking and Selection, Large Deviations, Pure Exploration},
}

\begin{document}
\maketitle

\begin{abstract}
We propose \textit{annealed entropic allocation}, an adaptive sampling policy based on an annealed, weighted soft-min formulation of static budget allocation. We replace the maximin large-deviation rate objective with a weighted log-sum-exp surrogate that blends challenger-specific pairwise scores through soft-min weights, avoiding hard switching when several challengers are nearly active. To capture tail behavior beyond the leading exponent, the surrogate incorporates saddlepoint prefactors from refined pairwise tail asymptotics. Because these corrections are subexponential, decreasing the annealing temperature with the budget preserves the same first-order target allocation. For the static problem, we prove uniform convergence to the hard minimum, concentration of soft-min weights on active challengers, and continuity of the induced target-allocation map under fixed weights. Experiments show that the proposed methods are consistently competitive: the no-saddlepoint ablation performs best in symmetric Gaussian and exponential slippage settings, while saddlepoint weighting can help in heterogeneous or asymmetric cases.
\end{abstract}


\section{INTRODUCTION}
\label{sec:intro}

Simulation is a key tool for the design and analysis of complex stochastic systems across domains ranging from supply chain management to energy systems. In many of these settings, the objective is to identify the best design from a finite set of alternatives. This problem is known as \emph{ranking and selection} (R\&S). Because high-fidelity simulation replications can be computationally expensive---for example, a single run of a large-scale power grid or turbine blade design model may take hours---the fundamental challenge is to allocate a limited simulation budget efficiently in order to maximize the probability of correct selection~(PCS).

Existing R\&S methods differ both in the type of guarantees they provide and in how they determine sequential allocations. Indifference-zone procedures~\citep{kim2001fully} provide finite-sample frequentist guarantees: they attain a target PCS whenever the best alternative is separated from the rest by at least a user-specified indifference parameter. Some procedures use asymptotic approximations to bound the probability of correct selection, then derive tractable two-stage or sequential allocation rules~\citep{chickNewTwoStageSequential2001} based on these bounds. Others allocate the next replication according to a one-step value-of-information~\citep{chick2010sequential,frazier2008knowledge}. Optimal computing budget allocation (OCBA) methods~\citep{chen2010stochastic} also use asymptotic approximations to derive target allocation ratios and sequentially track them.

Large-deviation (LD) analysis further sharpens this asymptotic perspective. For static allocations, \citet{glynn2004large} characterized the exponential decay rate of the false-selection probability and showed that rate-optimal sampling proportions solve a maximin problem over pairwise LD rates. \citet{gaoNewBudgetAllocation2017a} extended this framework to the expected-opportunity-cost criterion. Beyond the leading exponent, \citet{shi2024series} developed a Bahadur--Rao-type series expansion for PCS and proposed the FCBA policy based on this refinement. In parallel, adaptive methods seek to approach the LD-optimal allocation without explicitly solving the static maximin problem. For example, \citet{chen2023bold} introduced the balancing optimal large deviations (BOLD), which tracks the LD-optimal allocation by maintaining the total- and individual-balance conditions. \citet{qin2025dual} used dual variables and an information-directed selection rule to track the Karush--Kuhn--Tucker conditions. Recent pure-exploration work, including the Frank-Wolfe self-play approach of \citet{liu2025pure} and references therein, addresses the same nonsmooth maximin structure from a game-based Frank-Wolfe perspective.

The LD perspective provides a natural first-order asymptotic objective. Let \(p\in\Delta^{K-1}\) denote a static allocation over the \(K\) designs, where \(\Delta^{K-1}:=\{p\in\mathbb{R}_+^K:\sum_i p_i=1\}\), and let \(C_{b,j}(p)\) denote the pairwise LD rate of the event that challenger \(j\) is incorrectly ranked above the true best design~\(b\). \citet{glynn2004large} solve the static maximin allocation problem
\begin{equation}
    \max_{p\in\Delta^{K-1}}\min_{j\neq b} \ C_{b,j}(p).
     \label{eq:hardmin_obj}
\end{equation}
This characterization is valuable because it identifies the correct first-order asymptotic target. The objective~\eqref{eq:hardmin_obj}, however, optimizes a nonsmooth hard minimum over pairwise LD rates. Existing approaches such as BOLD, dual tracking, and Frank-Wolfe Self-Play retain this hard-min structure, and manage the resulting nonsmoothness through their update mechanisms. In contrast, this work proposes to handle the nonsmoothness by replacing the hard minimum with an annealed entropic soft-min surrogate. This distinction is important when several challengers have comparable, near-minimal rates, because small estimation errors or finite-sample fluctuations can change the apparent least-favorable challenger, causing abrupt shifts in the target allocation and potentially oscillatory adaptive sampling behavior. Moreover, in finite-budget regimes, first-order LD rates may be too coarse to discriminate reliably among such competing challengers. These considerations motivate the annealed entropic soft-min allocation objective, with saddlepoint prefactors used as a finite-budget refinement of the challenger weights when lower-order pairwise tail information is informative. The entropic weights distribute attention across near-active challengers, while annealing gradually reduces the smoothing parameter so that the surrogate approaches the hard-min behavior of the classical maximin formulation. Our main contributions are as follows.
\begin{itemize}
    \item We introduce an annealed, weighted soft-min surrogate for the maximin LD objective, replacing the hard minimum over challengers with a smooth, temperature-controlled aggregation. We set the challenger weights using saddlepoint prefactors from refined pairwise tail asymptotics; because these prefactors are subexponential, they refine finite-budget weighting among near-active challengers without changing the first-order allocation target.
  
    \item We establish key structural properties of the surrogate: it recovers the hard minimum under annealing, its Gibbs weights concentrate on the active challengers, and, for fixed challenger weights, the induced target-allocation map is continuous on the interior of the simplex.

    \item We develop \textit{annealed entropic allocation} (AEA), an adaptive policy that turns the surrogate target-allocation map into a sequential procedure by using plug-in estimates and cumulative tracking. 
  
    \item We evaluate AEA and an ablated no-saddlepoint variant on Gaussian and exponential instances. The ablation is particularly strong in symmetric slippage instances, where the true saddlepoint prefactor ratios across challengers are equal at the symmetric target; AEA is strongest in heterogeneous or asymmetric cases, where the prefactor carries genuine lower-order information.
\end{itemize}

\section{THEORETICAL FRAMEWORK}
\label{sec:theory}

\subsection{Problem Formulation and the Large-Deviation Objective}

Consider $K$ stochastic alternatives $\mathcal{X}=\{1,\dots,K\}$. For each alternative $i\in\mathcal{X}$, observations are drawn independently from a one-parameter exponential family with density or mass function
\begin{equation}
    f_i(y;\theta_i)
    \;=\;
    \exp\!\big(y\,\theta_i - A_i(\theta_i) + c_i(y)\big),
    \label{eq:ef_model}
\end{equation}
where $\theta_i\in\Theta_i\subset\mathbb{R}$ is the natural parameter and $A_i$ is twice continuously differentiable on the interior of $\Theta_i$, with $A_i''(\theta_i) > 0$. The mean of alternative $i$ is $\mu_i = A_i'(\theta_i)$. Common distributions encompassed by \eqref{eq:ef_model} include Gaussian distributions with known variance, Bernoulli, Poisson, and exponential distributions. Let $b\in\mathcal{X}$ denote the true best alternative, so that $\mu_b > \mu_j, \forall\, j\neq b.$

We begin with the static allocation problem in the LD framework. Let $p=(p_1,\dots,p_K)\in\Delta^{K-1}=\{p\in\mathbb{R}_+^K:\sum_{i=1}^K p_i=1\}$ denote the target sampling proportions, where $p_i$ is the fraction of budget allocated to alternative $i$. For a static allocation, fixed-budget LD analysis shows that the probability of incorrect selection is governed at first order by the most difficult challenger, which leads to \eqref{eq:hardmin_obj}. To define the pairwise rate $C_{b,j}(p)$, let
$
    d_i(\theta,\vartheta)
    =
    A_i(\vartheta) - A_i(\theta) - A_i'(\theta)(\vartheta - \theta)
$
denote the Kullback--Leibler divergence from the distribution with natural parameter $\theta$ to that with parameter $\vartheta$ under the exponential family of alternative $i$. For $j\neq b$, define the \emph{Cramér projection}
\begin{equation}
    \vartheta^j
    \;\in\;
    \arg\min_{\vartheta:\, A_b'(\vartheta_b)\,\le\, A_j'(\vartheta_j)}
    \;\sum_{i=1}^K p_i\, d_i(\vartheta_i,\theta_i),
    \label{eq:Cramer_projection}
\end{equation}
which identifies the least-favorable parameter configuration under which challenger $j$ appears at least as good as the best alternative $b$. The pairwise LD rate is then
\begin{equation}
    C_{b,j}(p)
    \;=\;
    \sum_{i=1}^K p_i\, d_i(\vartheta_i^j,\theta_i).
    \label{eq:general_rate}
\end{equation}
For each fixed $\vartheta$, the map $p\mapsto\sum_{i} p_i\,d_i(\vartheta_i,\theta_i)$ is linear; since $C_{b,j}(p)$ is the infimum of such linear functions over the constraint set in \eqref{eq:Cramer_projection}, it is concave in $p$. Under standard regularity conditions ensuring that the projection is unique for $p$ in the interior of $\Delta^{K-1}$, the envelope theorem gives
\begin{equation}
    \frac{\partial C_{b,j}(p)}{\partial p_i}
    \;=\;
    d_i(\vartheta_i^j,\theta_i),
    \label{eq:envelope}
\end{equation}
so $\partial C_{b,j}(p)/\partial p_i$ is the marginal contribution of alternative $i$ to the pairwise error exponent against challenger $j$. Substituting \eqref{eq:envelope} into \eqref{eq:general_rate} yields
\[
    C_{b,j}(p)
    =
    \sum_{i=1}^K p_i\,\frac{\partial C_{b,j}(p)}{\partial p_i}.
\]

Importantly, each pairwise rate $C_{b,j}$ is smooth on the interior of the simplex; the non-smoothness in \eqref{eq:hardmin_obj} arises solely from the outer minimum over challengers. In the sequential allocation, estimation noise can further amplify active-set switching and lead to unstable allocation updates. This motivates the annealed soft-min framework developed next.

\subsection{The Annealed Weighted Soft-Min Surrogate}
\label{sec:theory_softmin}

For annealing level $t>0$ and $p$, define the \emph{annealed weighted soft-min surrogate} of objective \eqref{eq:hardmin_obj} by
\begin{equation}
    F_t(p)
    =
    -\frac{1}{t}
    \log\!\left(
        \sum_{j\neq b}
        \rho_{j,t}\,\exp\!\big(-t\, C_{b,j}(p)\big)
    \right),
    \label{eq:softmin}
\end{equation}
where $\rho_{j,t}>0$ are prescribed challenger weights whose choice is deferred to Section~\ref{sec:saddlepoint}. Here, \(t\) plays the role of an inverse-temperature parameter; equivalently, the corresponding temperature is \(1/t\), which decreases to zero as the budget grows. Equivalently,
\[
    F_t(p)
    =
    -\frac{1}{t}
    \log\!\left(
        \sum_{j\neq b}
        \exp\!\left\{
            -t\!\left(
                C_{b,j}(p) - \frac{1}{t}\log\rho_{j,t}
            \right)
        \right\}
    \right),
\]
so $\rho_{j,t}$ acts as lower-order corrections $-t^{-1}\log\rho_{j,t}$ to the pairwise exponents. For finite $t$, challengers with comparable adjusted rates all contribute non-negligibly to the exponential aggregation; as $t\to\infty$, the sum becomes dominated by the least favorable challenger. Thus \eqref{eq:softmin} replaces the hard outer minimum in \eqref{eq:hardmin_obj} by a smooth surrogate while preserving the same first-order LD target. Only the relative magnitudes of the weights matter for the induced target allocation. Multiplying all $\rho_{j,t}$ by the same positive factor $c_t$ changes $F_t$ by the additive term $-(1/t)\log c_t$, which is independent of $p$, and leaves the Gibbs weights defined below unchanged. This observation will be useful in Section~\ref{sec:saddlepoint}, where challenger-independent factors in the saddlepoint prefactor can be dropped without changing the target allocation.

For the fixed-weight asymptotic theory, the requirement is that the weights be subexponential:
\begin{equation}
    \max_{j\neq b}\big|\log \rho_{j,t}\big|
    =
    o(t).
    \label{eq:subexp}
\end{equation}
Under \eqref{eq:subexp}, for each fixed $p$ we have
$F_t(p)\to \min_{j\neq b} C_{b,j}(p)$ as $t\to\infty$. Because the number of challengers is finite, the convergence is uniform on $\Delta^{K-1}$:
\[
    \sup_{p\in\Delta^{K-1}}
    \left|
        F_t(p)-\min_{j\neq b}C_{b,j}(p)
    \right|
    \le
    \max_{j\neq b}\left|\frac{1}{t}\log\rho_{j,t}\right|
    + \frac{\log(K-1)}{t}.
\]
For \(p\)-dependent saddlepoint weights, convergence is pointwise for each fixed interior \(p\), and uniform under any allocation floor. Hence $\rho_{j,t}$ influences the finite-budget interpolation among challengers but does not alter the LD objective. 

For any $t$, if the weights are constants with respect to $p$, then the map $p\mapsto F_t(p)$ is well behaved on the interior of $\Delta^{K-1}$. If each pairwise rate $C_{b,j}(p)$ is concave and continuously differentiable in $p$, then $F_t$ is also concave and continuously differentiable, because the map
$p \mapsto \log \rho_{j,t} - t\,C_{b,j}(p)$ is convex. The log-sum-exp of convex functions is again convex, and hence $F_t$ is concave and continuously differentiable.

Differentiating \eqref{eq:softmin} with respect to $p_i$ gives
\begin{equation}
\frac{\partial F_t}{\partial p_i}(p)
=
\sum_{j\neq b}q_{j,t}(p)\frac{\partial C_{b,j}(p)}{\partial p_i},
\qquad
q_{j,t}(p)
=
\frac{\rho_{j,t}\exp\{-tC_{b,j}(p)\}}
{\sum_{k\neq b}\rho_{k,t}\exp\{-tC_{b,k}(p)\}}.
    \label{eq:softmin_gradient}
\end{equation}
$q_{j,t}(p)$ is the Gibbs weight on challenger $j$. The coefficients \(q_{j,t}(p)\) place more mass on challengers with smaller adjusted rates
$
C_{b,j}(p)-t^{-1}\log \rho_{j,t}.
$
Because these weights are nonnegative and sum to one, \(\nabla F_t(p)\) is a convex combination of the pairwise gradients \(\nabla C_{b,j}(p)\). This is the key smoothing effect of the soft minimum: instead of switching abruptly from one active challenger to another, as the hard minimum does, the weights shift continuously across challengers for every finite \(t\).

\subsection{Challenger-Specific Allocation and Target Allocation}
\label{sec:theory_target}

For $j\neq b$, define the challenger-specific allocation vector
\begin{equation}
    h_{i,j}(p)
    =
    \frac{p_i\,d_i(\vartheta_i^j,\theta_i) }{C_{b,j}(p)}
    =
    \frac{p_i\,\partial_i C_{b,j}(p)}{C_{b,j}(p)},
    \label{eq:hij}
\end{equation}
where the second equality uses the envelope property \eqref{eq:envelope}. Since
$ d_i(\vartheta_i^j,\theta_i)\ge 0$, and
$
    \sum_{i=1}^K p_i\,d_i(\vartheta_i^j,\theta_i)  = C_{b,j}(p),
$
we have $h_{i,j}(p)\ge 0$ and
$\sum_{i=1}^K h_{i,j}(p)=1$, so $h_j(p)\in\Delta^{K-1}$. For fixed weights $\rho_{j,t}$, with Gibbs weights $q_{j,t}(p)$ defined in \eqref{eq:softmin_gradient}, define 
\begin{equation}
    \alpha_{j,t}(p)
    =
    \frac{q_{j,t}(p)C_{b,j}(p)}
    {\sum_{k\neq b}q_{k,t}(p)C_{b,k}(p)} ,
    \qquad j\neq b.
    \label{eq:alpha_soft}
\end{equation}
The weighted soft-min target allocation map is then
\begin{equation}
    \Pi_{i,t}(p)
    =
    \frac{
        \sum_{j\neq b} q_{j,t}(p)\, C_{b,j}(p)\, h_{i,j}(p)
    }{
        \sum_{j\neq b} q_{j,t}(p)\, C_{b,j}(p)
    },
    \qquad i=1,\dots,K.
    \label{eq:pi_soft}
\end{equation}
Because the coefficients $q_{j,t}(p)\,C_{b,j}(p)$ are nonnegative and the denominator normalizes their sum, $\Pi_t(p)$ is a convex combination of the simplex vectors $h_j(p)$, and therefore $\Pi_t(p)\in\Delta^{K-1}$.
Using \eqref{eq:hij} and \eqref{eq:softmin_gradient},
\[
    p_i\,\partial_i F_t(p)
    =
    \sum_{j\neq b}
    q_{j,t}(p)\,p_i\,\partial_i C_{b,j}(p)
    =
    \sum_{j\neq b}
    q_{j,t}(p)\,C_{b,j}(p)\,h_{i,j}(p).
\]
Summing over $i$ gives
\[
    \sum_{k=1}^K p_k\,\partial_k F_t(p)
    =
    \sum_{j\neq b}
    q_{j,t}(p)\, C_{b,j}(p).
\]
Hence \eqref{eq:pi_soft} can be written equivalently as
\[
    \Pi_{i,t}(p)
    =
    \frac{
        p_i\,\partial_i F_t(p)
    }{
        \sum_{k=1}^K p_k\,\partial_k F_t(p)
    }.
\]
Thus the target allocation has two interpretations: it is a soft aggregation of challenger-specific allocation vectors, and it is also the normalized vector of current shares times marginal surrogate gains. When one challenger dominates the Gibbs weights, $\Pi_t(p)$ is close to that challenger's pairwise allocation $h_j(p)$; when several challengers are nearly active, $\Pi_t(p)$ averages them smoothly. 

The weighted soft-min construction is closely related to the stationarity structure of \eqref{eq:hardmin_obj}, but differs in how the target is produced and tracked. \citet{chen2023bold} track balance conditions, while \citet{qin2025dual} track dual stationarity of the nonsmooth problem, where $p_i^* = \sum_{j\neq b}\nu_j h_{i,j}(p^*)$ and \(\nu_j=0\) for inactive challengers. \citet{liu2025pure} also works directly with the hard-min objective by maintaining a mixed skeptic over challengers and analyzing the resulting dynamics. In contrast, our work replaces the hard minimum by the smooth \(F_t\), yielding the explicit target map \(\Pi_t(p)\), and then follows this target by cumulative tracking. The hard active-set multipliers are replaced by finite-\(t\) soft coefficients, which are positive for all challengers in the simplex interior and concentrate on the lowest-rate challengers as \(t\to\infty\).

A useful consequence of the pairwise structure is that under-sampled designs become rate-limiting. Since the constraint in \eqref{eq:Cramer_projection} involves only $b$ and $j$, the optimizer sets $\vartheta_i^j=\theta_i$ for $i\notin\{b,j\}$, so only the pair $(b,j)$ contributes to $C_{b,j}(p)$. In the exponential families considered here, one can choose a feasible $\bar\vartheta_j$ with $A_j'(\bar\vartheta_j)\ge \mu_b$, yielding
\[
    C_{b,j}(p) \le p_j\,d_j(\bar\vartheta_j,\theta_j).
\]
Hence $C_{b,j}(p)\to 0$ as $p_j\downarrow 0$; similarly, $C_{b,j}(p)\to 0$ as $p_b\downarrow 0$. Thus, if the sampling share of a relevant design becomes too small, the corresponding pairwise rate collapses and that design becomes active.

\subsection{Second-Order Saddlepoint Weights}
\label{sec:saddlepoint}
 
The remaining ingredient is the finite-budget choice of the weights $\rho_{j,t}$. We use a pairwise saddlepoint tail approximation. Because these weights depend on $p$ and on the parameter, the structural results for $F_t$ and $\Pi_t$ are stated under the convention that, when the map is analyzed as a function of $p$, the challenger weights are held fixed. In the adaptive policy, these weights are recomputed from the current plug-in estimates and empirical allocation before forming the next target. 

Under allocation \(p\), design \(i\) receives \(tp_i\) samples. For a challenger $j\neq b$, the pairwise error event is that the sample mean of \(j\) exceeds that of the true best design. Thus define
\[
    T_{j,t}:=\bar X_{b,tp_b}-\bar X_{j,tp_j},
    \qquad
    E_j=\{T_{j,t}\le0\},
\]
where \(\bar X_{i,tp_i}:=(tp_i)^{-1}\sum_{\ell=1}^{tp_i}X_{i,\ell}\). It is convenient to work with the scaled statistic \(S_{j,t}:=tT_{j,t}\). For \(s\) in the moment-generating-function domain, define the scaled cumulant-generating function \(\kappa_j\) by
\[
    \log \mathbb{E}\!\left[e^{sS_{j,t}}\right]
    =
    t\,\kappa_j(s;\,p,\theta).
\]
Let \(\hat s_j=\hat s_j(p,\theta)\) denote the saddlepoint for the threshold zero, i.e.,
\[
    \kappa_j'(\hat{s}_j;\,p,\theta)=0.
\]
Since \(\kappa_j'(0;\,p,\theta)=\mu_b-\mu_j>0\) while \(E_j=\{S_{j,t}\le0\}\) is a left-tail event, the saddlepoint satisfies \(\hat s_j<0\). The LD rate is
\(C_{b,j}(p)=-\kappa_j(\hat{s}_j;\,p,\theta)\). Under standard light-tail regularity conditions~\citep{dembo2009large}, we have
\begin{equation}
    \mathbb{P}(E_j)
    \sim
    \rho_{j,t}^{\mathrm{sp}}\exp\{-tC_{b,j}(p)\},
    \qquad
    \rho_{j,t}^{\mathrm{sp}}
    :=
    \frac{1}{|\hat s_j|\,\sqrt{2\pi\, t\,\kappa_j''(\hat s_j;\,p,\theta)}} .
    \label{eq:sp_proportional}
\end{equation}
Only challenger weight ratios enter the Gibbs normalization~\eqref{eq:softmin_gradient}, so challenger-independent factors such as $(2\pi t)^{-1/2}$ may be omitted in implementation. With or without these common factors, for each fixed interior \(p\) the weights are subexponential in \(t\), so the first-order LD target is unchanged.

In common cases, including the Gaussian and exponential settings used in our experiments, the saddlepoint prefactor has closed-form expressions. Retaining the common \(t^{-1/2}\) factor for consistency with \eqref{eq:sp_proportional}, and omitting only the challenger-independent factor \((2\pi)^{-1/2}\), we have
\[
\text{Gaussian:}\quad
\frac{\sqrt{\sigma_b^2/p_b+\sigma_j^2/p_j}}{(\mu_b-\mu_j)\sqrt{t}},
\qquad
\text{Exponential:}\quad
\frac{\mu_jp_b+\mu_bp_j}{(\mu_b-\mu_j)\sqrt{t\,p_bp_j(p_b+p_j)}}.
\]
When no closed form is available, the same prefactor can still be computed with little overhead. Since \(\kappa_j''(s)>0\), the derivative \(\kappa_j'(s)\) is strictly increasing on its domain, so the saddlepoint \(\hat s_j\) is obtained from the one-dimensional equation \(\kappa_j'(\hat s_j)=0\). This root solve can be warm-started across iterations.

To see why this correction matters, consider the Gibbs weight ratio between two challengers $j$ and~$k$,
\begin{equation}
    \frac{q_{j,t}(p)}{q_{k,t}(p)}
    =
    \frac{\rho_{j,t}^{\mathrm{sp}}}{\rho_{k,t}^{\mathrm{sp}}}\,
    \exp\!\big(-t\,(C_{b,j}(p)-C_{b,k}(p))\big).
    \label{eq:gibbs_ratio_sp}
\end{equation}
When \(C_{b,j}(p) \approx C_{b,k}(p)\), the exponential factor is close to one, so the saddlepoint ratio determines the relative Gibbs mass. A challenger with larger $\rho_{j,t}^{\mathrm{sp}}$ has a larger refined tail approximation, receives more Gibbs mass in the soft-min target, and therefore has greater influence on the plug-in target allocation. Thus, when the pre-exponential factors differ meaningfully across near-active challengers, the saddlepoint correction provides a finite-\(t\) tie-breaker precisely in the regime where the hard minimum is most sensitive to small changes in the pairwise rates. 

\section{ANNEALED ENTROPIC ALLOCATION}
\label{sec:algorithm}

AEA is an adaptive plug-in policy based on the weighted soft-min framework. In the implementation, the annealing level in \eqref{eq:softmin} is not tuned separately but is set equal to the elapsed simulation budget. Let $t_0:=Kn_0$ denote the initialization budget, and let $t\ge t_0$ be the total number of samples collected so far. At each round, AEA computes plug-in rates and saddlepoint prefactors at the current allocation, and then evaluates the plug-in target allocation map~\eqref{eq:pi_soft}. At finite $t$, the Gibbs weights retain mass on near-active challengers, while as $t$ grows, equivalently as the temperature $1/t$ decreases to zero, the soft-min sharpens and the surrogate approaches the LD target. Other temperature scalings would have the same first-order limit but are left for future tuning studies.

Let $N_{i,t}$ be the number of samples allocated to alternative $i$ after $t$ total samples, $N_t=(N_{1,t},\dots,N_{K,t})$, and $p_t=N_t/t$. To convert the target sequence $\{\pi_t\}_{t\ge t_0}$ into a sampling rule, we use \emph{cumulative tracking}. Let $W_t=(W_{1,t},\dots,W_{K,t})$ denote the cumulative target counts, initialized at $W_{t_0}=N_{t_0}$. For $t\ge t_0$, set $W_{t+1}=W_t+\pi_t$ and sample $I_{t+1}\in\arg\max_i\{W_{i,t+1}-N_{i,t}\}$. Thus AEA selects the alternative whose realized count lags furthest behind its cumulative target. Algorithm~\ref{alg:AEA} summarizes the full procedure. The following proposition gives the tracking guarantee.

\begin{proposition}
\label{prop:tracking}
Under the cumulative-tracking recursion above, for all $t\ge t_0$,
$
\|N_t-W_t\|_\infty \le K-1.
$
Consequently, for all \(t>t_0\),
\[
\left\|
\frac{N_t-N_{t_0}}{t-t_0}
-\frac{1}{t-t_0}\sum_{\ell=t_0}^{t-1}\pi_\ell
\right\|_\infty
\le \frac{K-1}{t-t_0}.
\]
In particular, if $\frac{1}{t-t_0}\sum_{\ell=t_0}^{t-1}\pi_\ell \to \bar\pi$, then $p_t=N_t/t \to \bar\pi$.
\end{proposition}

\begin{proof}
Let \(D_t:=N_t-W_t\), with components \(D_{i,t}=N_{i,t}-W_{i,t}\). Since \(D_{t_0}=0\) and the update
\[
D_{i,t+1}=D_{i,t}-\pi_{i,t}+\mathbf{1}\{i=I_{t+1}\}
\]
preserves the sum because \(\sum_i \pi_{i,t}=1\), we have \(\sum_i D_{i,t}=0\) for all \(t\ge t_0\).
We claim that \(D_{i,t}<1\) for all \(i,t\). This is true at \(t=t_0\). If \(i\neq I_{t+1}\), then
\(D_{i,t+1}=D_{i,t}-\pi_{i,t}\le D_{i,t}<1\). If~\(i=I_{t+1}\), then, by the selection rule,
\[
D_{I_{t+1},t+1}=1-\max_j\bigl(-D_{j,t}+\pi_{j,t}\bigr)<1,
\]
since the \(K\) terms \(-D_{j,t}+\pi_{j,t}\) sum to \(1\), so their maximum is at least \(1/K>0\). Hence \(D_{i,t}<1\) for \(i,t\). Together with \(\sum_i D_{i,t}=0\), this gives
\[
-(K-1)<D_{i,t}=-\sum_{j\ne i}D_{j,t}<1,
\]
and therefore
\(\|N_t-W_t\|_\infty\le K-1\).
Finally, \(W_t=N_{t_0}+\sum_{\ell=t_0}^{t-1}\pi_\ell\), so
\[
\left\|
\frac{N_t-N_{t_0}}{t-t_0}
-\frac{1}{t-t_0}\sum_{\ell=t_0}^{t-1}\pi_\ell
\right\|_\infty
=
\frac{\|N_t-W_t\|_\infty}{t-t_0}
\le \frac{K-1}{t-t_0}.
\]
If \((t-t_0)^{-1}\sum_{\ell=t_0}^{t-1}\pi_\ell\to\bar\pi\), then
\((N_t-N_{t_0})/(t-t_0)\to\bar\pi\). Since
\[
p_t=N_t/t=N_{t_0}/t+((t-t_0)/t)\,(N_t-N_{t_0})/(t-t_0),
\]
it follows that
\(p_t\to\bar\pi\).
\end{proof}

Proposition~\ref{prop:tracking} shows that the realized allocation follows the cumulative sequence of plug-in targets up to an $O(1/t)$ error in proportions. It does not establish convergence of the adaptive policy to the LD-optimal allocation or exponential PICS decay. A full adaptive fixed-budget analysis is left for future work.

\begin{algorithm}[t]
\caption{Annealed Entropic Allocation (AEA)}
\label{alg:AEA}
\begin{algorithmic}[1]
\REQUIRE Budget $T$, initial replications per alternative $n_0$
\STATE Sample each alternative $n_0$ times; set $t \leftarrow Kn_0$ and $W_t \leftarrow N_t$
\FOR{$t=Kn_0,\dots,T-1$}
    \STATE Update plug-in estimates $\hat\theta_{i,t}$ and empirical allocation $p_t = N_t/t$
    \STATE Set $b_t \leftarrow \arg\max_i A_i'(\hat\theta_{i,t})$ 
    \FOR{each challenger $j\neq b_t$}
        \STATE Compute plug-in Cramér projection $\hat\vartheta_t^j$ via \eqref{eq:Cramer_projection}
        \STATE Compute rate $\hat C_{b_t,j}(p_t)$ via \eqref{eq:general_rate} and allocation vector $\hat h_j(p_t)$ via \eqref{eq:hij}
        \STATE Compute the saddlepoint prefactor $\hat\rho_{j,t}$ via \eqref{eq:sp_proportional}; set log-weight: $w_j \leftarrow \log \hat\rho_{j,t} - t\,\hat C_{b_t,j}(p_t)$
    \ENDFOR
    \STATE Compute plug-in Gibbs weights and target via \eqref{eq:softmin_gradient} and \eqref{eq:pi_soft}; set $W_{t+1}\leftarrow W_t+\pi_t$
    \STATE Select and sample $I_{t+1} \leftarrow \arg\max_i(W_{i,t+1}-N_{i,t})$; update statistics and counts
\ENDFOR
\RETURN $\arg\max_i A_i'(\hat\theta_{i,T})$
\end{algorithmic}
\end{algorithm}

\section{ASYMPTOTIC PROPERTIES}
\label{sec:analysis}

This section establishes the first-order static asymptotic properties of the weighted soft-min objective evaluated at the true parameter.

\subsection{Preservation of the First-Order Large-Deviation Target}

\begin{proposition}
\label{prop:first_order}
Let the positive weights \(\rho_{j,t}\) be independent of \(p\), let
$L_t := \max_{j\neq b}|\log\rho_{j,t}|$, and assume $L_t = o(t)$. Then
\begin{equation}
    \sup_{p\in\Delta^{K-1}} |F_t(p) - m(p)|
    \le
    \frac{L_t + \log(K-1)}{t}
    \to 0,
    \label{eq:uniform_bound}
\end{equation}
where $m(p) := \min_{j\neq b} C_{b,j}(p).$
If each $C_{b,j}$ is continuous on $\Delta^{K-1}$ and the allocation problem in \eqref{eq:hardmin_obj} admits a unique optimizer, then any sequence of maximizers of $F_t$ over $\Delta^{K-1}$ converges to the LD-optimal allocation.
\end{proposition}

\begin{proof}
Rewriting $F_t$ gives
\[
    F_t(p)
    =
    m(p)
    -
    \frac{1}{t}
    \log\!\sum_{j\neq b}
    \rho_{j,t}\,
    \exp\!\big(-t\,(C_{b,j}(p)-m(p))\big).
\]
Since $C_{b,j}(p)-m(p)\ge 0$ for every $j$ and equality holds for at least one challenger,
\[
    \min_{j\neq b}\rho_{j,t}
    \le
    \sum_{j\neq b}\rho_{j,t}\,
    \exp\!\big(-t\,(C_{b,j}(p)-m(p))\big)
    \le
    (K-1)\,\max_{j\neq b}\rho_{j,t}.
\]
Taking $-t^{-1}\log(\cdot)$ and using $|\log\rho_{j,t}|\le L_t$ yields \eqref{eq:uniform_bound}. The convergence of maximizers then follows from uniform convergence on the compact set $\Delta^{K-1}$, continuity of $m$, and uniqueness of optimal allocation.
\end{proof}

Proposition~\ref{prop:first_order} confirms that saddlepoint prefactors affect only the finite-\(t\) weighting of challengers in the surrogate: as long as they remain sub-exponential, they do not alter the limiting first-order target.

\subsection{Regularity of the Target Allocation Map}

The following result assumes that the weights $\rho_{j,t}$ are fixed positive constants independent of $p$.

\begin{proposition}
\label{prop:continuity}
Fix $t<\infty$ and positive weights $\rho_{j,t}$. Suppose each $C_{b,j}(p)$ is continuously differentiable and strictly positive on the interior of $\Delta^{K-1}$, and satisfies $\sum_{i=1}^K p_i\,\partial_i C_{b,j}(p)=C_{b,j}(p)$. Then $p\mapsto \Pi_t(p)$ is continuous on the interior of $\Delta^{K-1}$. If, in addition, each $C_{b,j}$ is twice continuously differentiable with bounded Hessian on a compact subset $\mathcal{K}$ of the interior of $\Delta^{K-1}$, then $p\mapsto \Pi_t(p)$ is Lipschitz on $\mathcal{K}$.
\end{proposition}

\begin{proof}
Since the weights are fixed, the gradient of $F_t$ is
\[
    \partial_i F_t(p)
    =
    \sum_{j\neq b} q_{j,t}(p)\,\partial_i C_{b,j}(p),
\]
which is continuous whenever $C_{b,j}$ is continuously differentiable. The denominator of $\Pi_t$ satisfies
\[
    \sum_{k=1}^K p_k\,\partial_k F_t(p)
    =
    \sum_{j\neq b} q_{j,t}(p)\, C_{b,j}(p)
    > 0.
\]
Thus continuity of $\Pi_t$ follows from the quotient representation. On a compact subset $\mathcal{K}$ of the interior, the denominator is continuous and strictly positive, hence bounded away from zero. The bounded-Hessian assumption implies that each $\nabla C_{b,j}$ is Lipschitz on $\mathcal{K}$; since $\mathcal{K}$ is compact, the gradients are also bounded, and the softmax weights $q_{j,t}(p)$ are Lipschitz functions of $p$. Hence $\nabla F_t=\sum_{j\neq b}q_{j,t}\nabla C_{b,j}$ is Lipschitz on $\mathcal{K}$, and the quotient formula for \(\Pi_t\) gives the Lipschitz claim.
\end{proof}

Proposition~\ref{prop:continuity} formalizes the main structural advantage of the entropic smoothing: the allocation target varies continuously with the underlying rate geometry. In contrast, hard minimum-based allocation rules can exhibit discontinuous jumps whenever the identity of the worst-case challenger changes.

\subsection{Concentration of Gibbs Weights}

The Gibbs weights \eqref{eq:softmin_gradient} asymptotically concentrate on the \emph{active challenger set}
\[
    \mathcal{A}(p)
    =
    \left\{
        j\neq b:
        C_{b,j}(p) = \min_{k\neq b}\, C_{b,k}(p)
    \right\}.
\]
For any non-active challenger $j\notin\mathcal{A}(p)$ and any active challenger $k\in\mathcal{A}(p)$,
\[
    \frac{q_{j,t}(p)}{q_{k,t}(p)}
    =
    \frac{\rho_{j,t}}{\rho_{k,t}}\,
    \exp\!\big(-t\,(C_{b,j}(p)-C_{b,k}(p))\big)
    \to 0,
\]
since $C_{b,j}(p)-C_{b,k}(p)>0$ and the weight ratio is subexponential. Hence $q_{j,t}(p)\to 0$ for all $j\notin\mathcal{A}(p)$, and $\sum_{j\in\mathcal{A}(p)} q_{j,t}(p)\to 1$. Thus the Gibbs weights asymptotically identify the same active challenger set as the hard minimum, while preserving a smooth finite-budget interpolation.

When the weights are chosen to be the saddlepoint prefactors from Section~\ref{sec:saddlepoint}, they do not change the limiting active set, because their ratios are subexponential. Their effect is finite-budget: among challengers with nearly equal rates, the ratio $\rho_{j,t}^{\mathrm{sp}}/\rho_{k,t}^{\mathrm{sp}}$ determines how the Gibbs mass is split. Taken together, the results show that the surrogate preserves the LD target and yields a smooth target map. Extending these statements to the adaptive policy remains an open direction for future work.

\section{NUMERICAL EXPERIMENTS}
\label{sec:experiments}

We evaluate AEA and its no-saddlepoint-prefactor ablation, AEA-noSP, on the Gaussian and exponential instances in Table~\ref{tab:test_instances}. AEA-noSP uses the same annealing schedule and cumulative-tracking rule as AEA, but sets \(\rho_{j,t}\equiv 1\) in the Gibbs weights, thereby isolating the effect of the annealed soft-min aggregation from the saddlepoint prefactor. Gaussian variances are assumed known. Exponential rewards are parameterized by their means, \(X_i \sim\mathrm{Exp}(\mathrm{mean}=\mu_i)\), so larger \(\mu_i\) indicates a better alternative. For the Gaussian instances, we benchmark against three LD-based methods--BOLD~\citep{chen2023bold}, OEA~\citep{gaoNewBudgetAllocation2017a}, and FCBA~\citep{shi2024series}--two top-two Thompson sampling variants--TS-KKT-IDS~\citep{qin2025dual} and TTTS ($\beta=1/2$)~\citep{russo2020simple}--as well as standard Thompson sampling~(TS), OCBA with most-starving allocation (OCBA-MSA)~\citep{chen2010stochastic}, and equal allocation. All TS-based methods use flat priors for Gaussian means and independent \(\mathrm{Gamma}(1,1)\) shape-rate priors on inverse means for exponential rewards. TS-KKT-IDS draws \(\tilde\mu\), sets \(b=\arg\max_i\tilde\mu_i\), chooses \(j=\arg\min_{k\ne b}C_{b,k}(p_t;\tilde\mu)\), and samples according to \(h_j(p_t;\tilde\mu)\); for exponential instances, \(C_{b,j}\) and \(h_j\) are computed using the fixed-budget KL direction \(d(x,\mu)=x/\mu-1-\log(x/\mu)\). For TTTS, we cap the number of challenger-sampling attempts at \(1{,}000\), since repeated posterior sampling becomes expensive as the posterior concentrates. 

Each algorithm is initialized with $5$ replications per alternative. For each instance, we run $50{,}000$ independent macro-replications. For each method and instance, we run independent macro-replications under the same initialization rule. At each sampling budget $t$, the recommended design is the design with the largest empirical mean. Performance is measured by the empirical probability of incorrect selection ($\mathrm{PICS}$). All plots report $\mathrm{PICS}$ against the budget $t$, with a logarithmic vertical axis; the curves are used as empirical comparisons rather than as proof of exponential decay.

\begin{table}[!t]
\centering
\caption{Gaussian and exponential test instances.}
\label{tab:test_instances}
\begingroup
\footnotesize
\setlength{\tabcolsep}{4pt}
\renewcommand{\arraystretch}{1.12}
\begin{tabularx}{\linewidth}{@{}l c >{\raggedright\arraybackslash}X@{}}
\toprule
\textbf{Instance} & \textbf{$K$} & \textbf{Specification} \\
\midrule
\multicolumn{3}{@{}l}{\textit{Gaussian instances}} \\
G1 
& 10 
& $\mu_i=i-1,\quad \sigma_i^2=36,\quad i=1,\ldots,10.$ \\
G2 
& 10 
& $\mu_1=1,\quad \mu_i=0\ \text{for } i=2,\ldots,10,\quad \sigma_i^2=36.$ \\
G3 
& 10 
& $\mu_i=\left((10-i)/4\right)^2,\quad 
   (\sigma_1^2,\ldots,\sigma_{10}^2)
   =(36,36,36,36,36,36,72,144,216,36).$ \\
G4 
& 100 
& $\mu_i=(i-1)/10,\quad \sigma_i^2=1,\quad i=1,\ldots,100.$ \\
\addlinespace[1pt]
\multicolumn{3}{@{}l}{\textit{Exponential instances}} \\
E1 
& 10 
& $\boldsymbol{\mu}
   =(1.00,0.95,0.90,0.84,0.78,0.70,0.62,0.56,0.52,0.48).$ \\
E2 
& 10 
& $\mu_1=1.00,\quad \mu_i=0.90\ \text{for } i=2,\ldots,10.$ \\
\bottomrule
\end{tabularx}
\endgroup
\end{table}

\subsection{Summary of Numerical Findings}

\begin{figure}[!t]
\centering
\begingroup
\setlength{\tabcolsep}{2pt}
\newcommand{\gpanel}[2]{%
  \begin{minipage}[t]{0.47\linewidth}
  \centering
  \includegraphics[width=\linewidth,trim={1pt 1pt 1pt 1pt},clip]{#2}\\
  {\footnotesize #1}
  \end{minipage}%
}
\begin{tabular}{@{}cc@{}}
\gpanel{(a) G1}{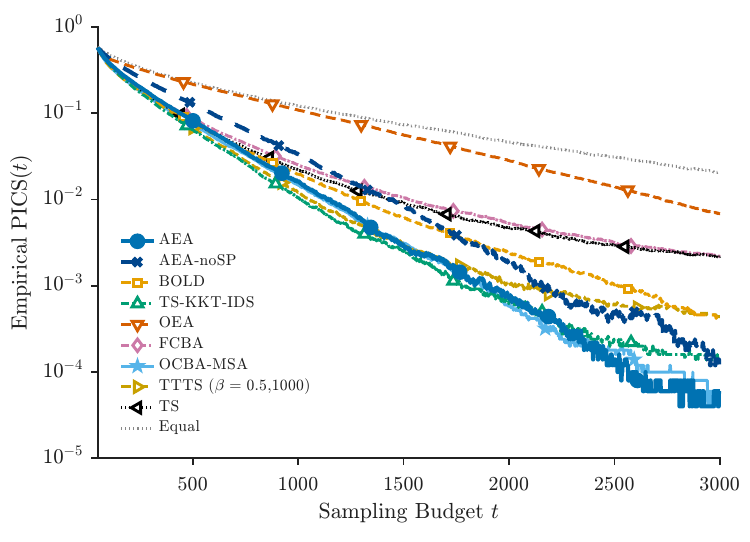}
&
\gpanel{(b) G2}{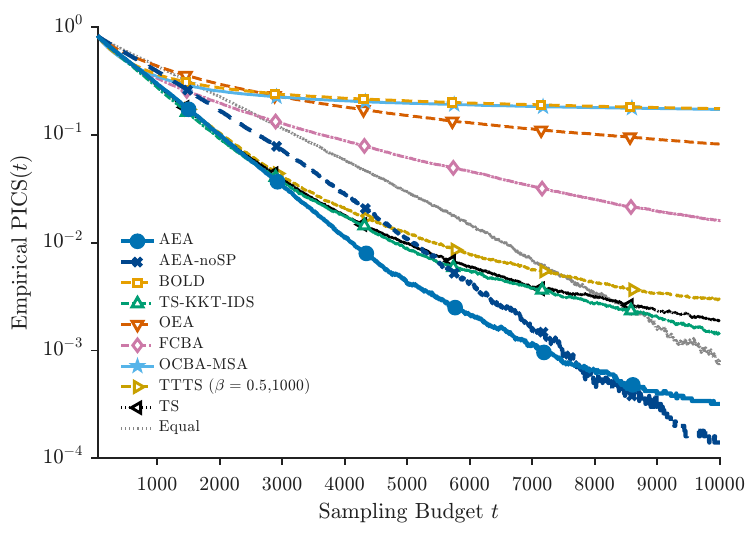}
\\
\gpanel{(c) G3}{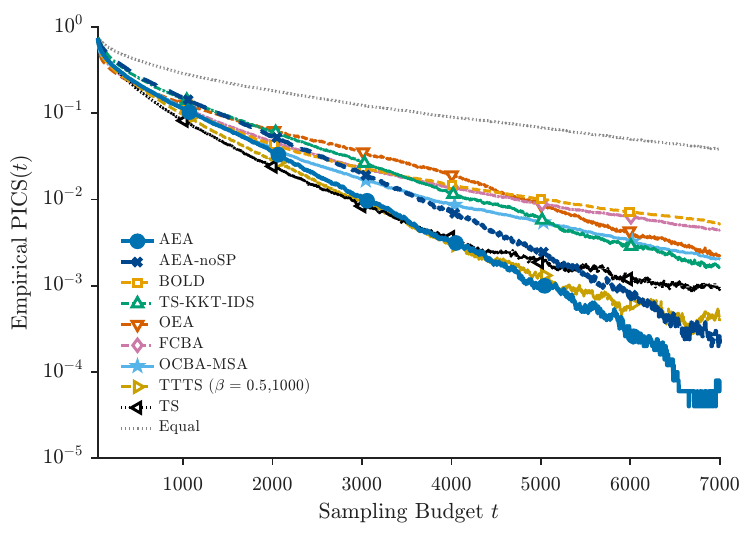}
&
\gpanel{(d) G4}{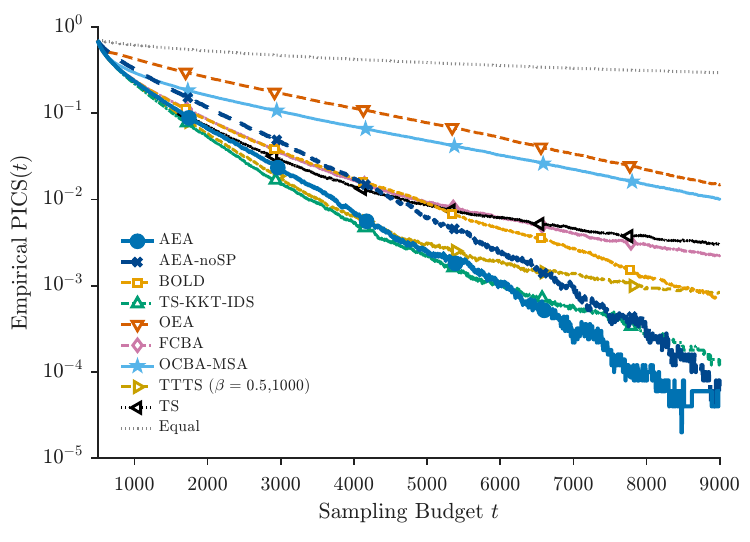}
\end{tabular}
\endgroup
\caption{Empirical PICS on Gaussian test instances.}
\label{fig:gaussian_results}
\end{figure}

Figures~\ref{fig:gaussian_results}--\ref{fig:exp_results} report the finite-budget PICS, Table~\ref{tab:pics_wilson} gives terminal PICS estimates with 95\% Wilson intervals, and Table~\ref{tab:cpu_time} reports computation times. On the log scale, nearly linear segments are LD-like, whereas slope changes reflect finite-budget transients and rare slow-recovery paths rather than clean exponential decay. Such slow recovery is most visible in the slippage instances G2 and E2, where unfavorable early estimates can make the empirical best design incorrect and delay reallocation. Nevertheless, the proposed soft-min methods remain the strongest overall: one of AEA or AEA-noSP attains the smallest terminal PICS estimate in five of the six instances, and on G1 the smallest point estimate is not significantly separated from the proposed methods. The Holm-adjusted tests in Table~\ref{tab:pics_wilson} show that AEA-noSP is significantly best on G2 and E2, full AEA is significantly best on G3 and E1, while G1 and G4 have no unique significant winner.

On the Gaussian instances, AEA and AEA-noSP are consistently in the leading group. On G1, OCBA-MSA has the smallest terminal estimate, but not significantly below AEA, AEA-noSP, or TS-KKT-IDS. On symmetric slippage G2, AEA-noSP achieves the smallest terminal PICS significantly, with AEA as the next-best method. Equal allocation becomes relatively competitive only at larger budgets: it is inefficient but never starves any alternative, so its static error probability has a clean exponential-decay interpretation; however, it remains worse than both proposed variants at the terminal budget. This supports annealed soft-min aggregation: for nearly exchangeable challengers, smoothing attention across near-active alternatives is more stable than hard active-challenger selection. Moreover, in G2 the true Gaussian prefactors are equal at equal challenger shares, so AEA-noSP avoids finite-budget plug-in prefactor imbalances rather than discarding useful challenger-specific information. On G3, where variances are heterogeneous, full AEA is significantly best, showing that the saddlepoint prefactor can help when near-active challengers differ in scale. On G4, AEA and AEA-noSP have essentially identical terminal performance, with TS-KKT-IDS close behind. Overall, the Gaussian results show that soft-min allocation is robust, while the prefactor is most valuable under genuine heterogeneity.

The exponential results show the same qualitative pattern. On E1, AEA attains the significantly lowest terminal PICS, improving over AEA-noSP, TS-KKT-IDS, TTTS, and the other baselines; this shows a substantial benefit from the saddlepoint prefactor in this asymmetric exponential instance. On the exponential slippage instance E2, AEA-noSP is significantly best and AEA is second, ahead of all non-proposed methods. Thus, in these symmetric slippage test instances, AEA-noSP benefits because the true prefactor ratios carry no challenger-specific information; in contrast, G3 and E1 show that saddlepoint weighting is beneficial when lower-order differences are genuine. The ablation therefore strengthens the contribution: AEA-noSP validates the annealed soft-min allocation mechanism, and full AEA demonstrates the additional value of saddlepoint weighting in heterogeneous or asymmetric cases.

Table~\ref{tab:cpu_time} shows that the proposed methods are also computationally competitive. AEA-noSP is faster than AEA, as expected, because it avoids computing the saddlepoint prefactors. AEA adds modest absolute overhead for computing the saddlepoint prefactors and remains much faster than BOLD and TTTS on most instances, while being comparable to OEA and TS-KKT-IDS. On the exponential instances, AEA and AEA-noSP are especially favorable: they are far faster than BOLD and TTTS and close to the fastest LD-based competitors. We omit wall-clock times for TS and equal allocation because their computational costs are negligible.

\begin{figure}[!t]
\centering
\begin{minipage}{0.48\linewidth}
    \centering
    \includegraphics[width=\linewidth,trim={1pt 1pt 1pt 1pt},clip]{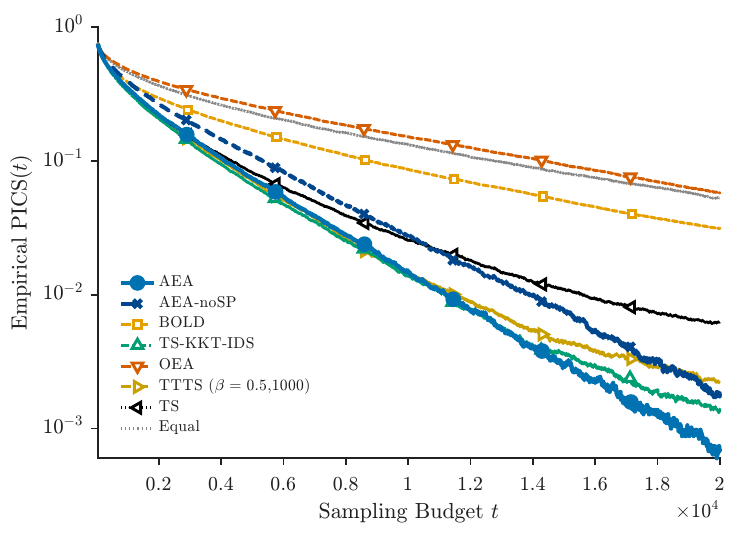}\\[-1pt]
    {\footnotesize (a) E1}
\end{minipage}
\hfill
\begin{minipage}{0.48\linewidth}
    \centering
    \includegraphics[width=\linewidth,trim={1pt 1pt 1pt 1pt},clip]{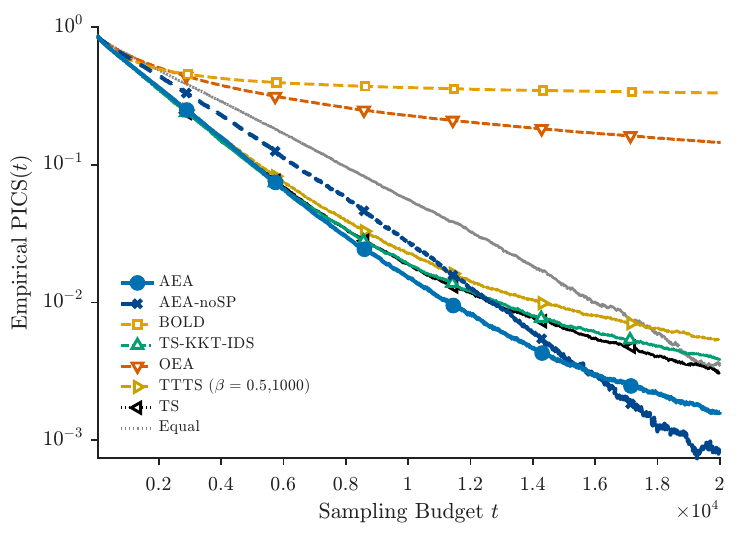}\\[-1pt]
    {\footnotesize (b) E2}
\end{minipage}
\caption{Empirical PICS on exponential test instances.}
\label{fig:exp_results}
\end{figure}

\begin{table}[!t]
\centering
\caption{Terminal empirical PICS. Values are percentages with 95\% Wilson confidence intervals.}
\label{tab:pics_wilson}
\begin{threeparttable}
\scriptsize
\setlength{\tabcolsep}{1pt}
\renewcommand{\arraystretch}{1.08}
\newcommand{\wci}[3]{\makecell[c]{#1\\[-1pt]\tiny [#2,#3]}}
\newcommand{\bestwci}[3]{\makecell[c]{\textbf{#1}\\[-1pt]\tiny [#2,#3]}}
\newcommand{\bestwcisig}[3]{\makecell[c]{\textbf{#1}$^{*}$\\[-1pt]\tiny [#2,#3]}}
\begin{tabular*}{\textwidth}{@{\extracolsep{\fill}}l*{10}{c}@{}}
\toprule
\textbf{Inst.}
& \textbf{AEA}
& \makecell{\textbf{AEA-noSP}}
& \textbf{BOLD}
& \makecell{\textbf{TS-KKT-IDS}}
& \textbf{OEA}
& \textbf{FCBA}
& \makecell{\textbf{OCBA-MSA}}
& \textbf{TTTS}
& \textbf{TS}
& \textbf{Equal} \\
\midrule
G1
& \wci{0.006}{0.002}{0.018}
& \wci{0.010}{0.004}{0.023}
& \wci{0.040}{0.026}{0.062}
& \wci{0.014}{0.007}{0.029}
& \wci{0.684}{0.616}{0.760}
& \wci{0.220}{0.183}{0.265}
& \bestwci{0.004}{0.001}{0.015}
& \wci{0.044}{0.029}{0.067}
& \wci{0.220}{0.183}{0.265}
& \wci{2.048}{1.928}{2.176}
\\
G2
& \wci{0.032}{0.020}{0.052}
& \bestwcisig{0.014}{0.007}{0.029}
& \wci{17.496}{17.165}{17.832}
& \wci{0.150}{0.120}{0.188}
& \wci{8.196}{7.959}{8.440}
& \wci{1.604}{1.498}{1.718}
& \wci{17.240}{16.911}{17.574}
& \wci{0.298}{0.254}{0.350}
& \wci{0.188}{0.158}{0.230}
& \wci{0.078}{0.057}{0.107}
\\
G3
& \bestwcisig{0.008}{0.003}{0.021}
& \wci{0.024}{0.014}{0.042}
& \wci{0.522}{0.463}{0.589}
& \wci{0.168}{0.136}{0.208}
& \wci{0.230}{0.192}{0.276}
& \wci{0.444}{0.389}{0.506}
& \wci{0.204}{0.168}{0.248}
& \wci{0.042}{0.027}{0.064}
& \wci{0.098}{0.074}{0.130}
& \wci{3.808}{3.644}{3.979}
\\
G4
& \bestwci{0.006}{0.002}{0.018}
& \bestwci{0.006}{0.002}{0.018}
& \wci{0.070}{0.050}{0.097}
& \wci{0.014}{0.007}{0.029}
& \wci{1.484}{1.382}{1.594}
& \wci{0.224}{0.186}{0.269}
& \wci{1.004}{0.920}{1.095}
& \wci{0.084}{0.062}{0.114}
& \wci{0.298}{0.254}{0.350}
& \wci{29.440}{29.042}{29.841}
\\
E1
& \bestwcisig{0.068}{0.049}{0.095}
& \wci{0.178}{0.145}{0.219}
& \wci{3.136}{2.987}{3.292}
& \wci{0.138}{0.109}{0.175}
& \wci{5.794}{5.593}{6.002}
& ---
& ---
& \wci{0.218}{0.181}{0.263}
& \wci{0.614}{0.549}{0.686}
& \wci{5.260}{5.068}{5.459}
\\
E2
& \wci{0.160}{0.129}{0.199}
& \bestwcisig{0.082}{0.060}{0.111}
& \wci{33.220}{32.808}{33.634}
& \wci{0.388}{0.337}{0.446}
& \wci{14.510}{14.204}{14.821}
& ---
& ---
& \wci{0.544}{0.483}{0.612}
& \wci{0.302}{0.258}{0.354}
& \wci{0.358}{0.309}{0.414}
\\
\bottomrule
\end{tabular*}
\begin{tablenotes}[flushleft]
\footnotesize
\item \textit{Notes.} Boldface = row minimum. 
$^{*}$ = significantly lower than all others 
(one-sided Fisher, Holm-adjusted, $\alpha=0.05$).
\end{tablenotes}
\end{threeparttable}
\end{table}

\begin{table}[!t]
\centering
\caption{Mean wall-clock time per macro-replication.}
\label{tab:cpu_time}
\begin{threeparttable}
\begingroup
\scriptsize
\setlength{\tabcolsep}{1pt}
\renewcommand{\arraystretch}{1.08}
\newcommand{\meanse}[2]{#1\,{\scriptsize(#2)}}
\newcommand{\bestmeanse}[2]{\textbf{#1\,{\scriptsize(#2)}}}
\begin{tabular*}{\linewidth}{@{\extracolsep{\fill}}lcccccccc@{}}
\toprule
\textbf{Inst.}
& \textbf{AEA}
& \textbf{AEA-noSP}
& \textbf{BOLD}
& \makecell{\textbf{TS-KKT-IDS}}
& \textbf{OEA}
& \textbf{FCBA}
& \makecell{\textbf{OCBA-MSA}}
& \textbf{TTTS} \\
\midrule
G1 
& \meanse{1.700}{0.009}
& \meanse{1.412}{0.025}
& \meanse{5.620}{0.014}
& \meanse{2.288}{0.004}
& \meanse{2.440}{0.002}
& \meanse{3.088}{0.003}
& \bestmeanse{1.242}{0.002}
& \meanse{17.052}{0.019} \\
G2 
& \meanse{3.474}{0.012}
& \meanse{2.760}{0.009}
& \meanse{11.775}{0.056}
& \meanse{4.314}{0.016}
& \meanse{4.364}{0.009}
& \meanse{5.493}{0.010}
& \bestmeanse{2.567}{0.013}
& \meanse{17.248}{0.057} \\
G3 
& \meanse{3.886}{0.013}
& \bestmeanse{1.282}{0.003}
& \meanse{36.469}{0.063}
& \meanse{21.663}{0.016}
& \meanse{11.767}{0.011}
& \meanse{13.928}{0.014}
& \meanse{3.020}{0.013}
& \meanse{37.346}{0.066} \\
G4 
& \meanse{31.816}{0.014}
& \meanse{26.408}{0.009}
& \meanse{50.727}{0.177}
& \meanse{30.325}{0.065}
& \bestmeanse{25.471}{0.013}
& \meanse{40.678}{0.014}
& \meanse{27.135}{0.023}
& \meanse{65.614}{0.179} \\
E1 
& \meanse{12.211}{0.013}
& \meanse{8.832}{0.008}
& \meanse{89.833}{0.090}
& \bestmeanse{7.888}{0.020}
& \meanse{13.058}{0.029}
& ---
& ---
& \meanse{366.555}{1.267} \\
E2 
& \meanse{11.135}{0.017}
& \meanse{8.899}{0.009}
& \meanse{57.521}{0.138}
& \bestmeanse{8.391}{0.021}
& \meanse{10.333}{0.015}
& ---
& ---
& \meanse{138.522}{0.562} \\
\bottomrule
\end{tabular*}
\endgroup
\begin{tablenotes}[flushleft]
\footnotesize
\item \textit{Notes.} Entries are mean wall-clock times in milliseconds; standard errors are shown in parentheses. Boldface indicates the row minimum. Measured in Julia 1.12.6 on an Apple M2 Pro.
\end{tablenotes}
\end{threeparttable}
\end{table}

\section{CONCLUSION}
\label{sec:conclusion}

We introduced annealed entropic allocation, an adaptive plug-in policy motivated by a large-deviations framework using a smooth surrogate of the hard maximin objective. By replacing hard active-challenger selection with an annealed weighted soft minimum, AEA gives near-active challengers positive finite-budget weight while still approaching the classical hard-min LD objective. Saddlepoint prefactors provide subexponential, lower-order challenger weights and therefore do not change the first-order LD target. Our theory establishes preservation of the static LD-optimal target, continuity of the fixed-weight target-allocation map, concentration of Gibbs weights on active challengers, and a cumulative-tracking guarantee. The no-saddlepoint ablation shows that soft-min aggregation is the main source of empirical robustness, while saddlepoint weighting can add gains in heterogeneous or asymmetric instances.

These results provide a static LD justification for the surrogate and a tracking guarantee for the adaptive implementation. Future work includes proving convergence and exponential PICS decay for the adaptive procedure, reporting historical allocation-ratio trajectories, studying slow recovery from early estimation error, and assessing forced-exploration or allocation-floor safeguards. Other directions include annealing and initialization sensitivity, probability-of-good-selection criteria, common-random-number settings, structured or continuous design spaces, input-data uncertainty, and fixed-confidence uses of the saddlepoint prefactor.

\bibliographystyle{unsrtnat}
\bibliography{references}

\end{document}